\def\year{2018}\relax
\documentclass[letterpaper]{article} 
\usepackage{aaai18}  
\usepackage{times}  
\usepackage{helvet}  
\usepackage{courier}  
\usepackage{url}  
\usepackage{graphicx}  
\usepackage{booktabs}       
\usepackage{amsfonts}       
\usepackage{amsmath}
\usepackage{nicefrac}       
\usepackage{microtype}      
\usepackage{algorithm}
\usepackage{algpseudocode}
\usepackage{amsthm}
\usepackage{graphicx}
\usepackage{caption}
\usepackage{subcaption}
\usepackage{xcolor}
\usepackage{xr}
\usepackage{natbib}
\usepackage{amssymb}

\newtheorem{theorem}{Theorem}
\newtheorem{corollary}{Corollary}
\newtheorem{lemma}{Lemma}
\newtheorem{observation}{Observation}
\newtheorem{definition}{Definition}

\newtheorem{innercustomlemma}{Lemma}
\newenvironment{customlemma}[1]
  {\renewcommand\theinnercustomlemma{#1}\innercustomlemma}
  {\endinnercustomlemma}

\DeclareMathOperator{\trace}{trace}

\newcommand{\exc}[2]{{\mathbb E}\left[ #1 \,\middle \vert\, #2 \right]}
\newcommand{\exs}[2]{{\mathbb E_{#1}}\left[ #2 \right]}
\newcommand{\vars}[2]{{\mathbb V_{#1}}\left[ #2 \right]}
\newcommand{\excs}[3]{{\mathbb E_{#1}}\left[ #2  \,\middle \vert\, #3 \right]}

\algnewcommand{\LineComment}[1]{\State \(\triangleright\) #1}

\newcommand{\Sigmasqr}[0]{\Sigma^{1/2}}

\begin{document}

\title{Expected Policy Gradients}

\author{
  Kamil Ciosek \and Shimon Whiteson\\
  Department of Computer Science, University of Oxford\\
  Wolfson Building,
  Parks Road,
  Oxford OX1 3QD \\
  \texttt{\{kamil.ciosek,shimon.whiteson\}@cs.ox.ac.uk}
}

\maketitle

\begin{abstract}

We propose \emph{expected policy gradients} (EPG), which unify stochastic policy gradients (SPG) and deterministic policy gradients (DPG) for reinforcement learning. Inspired by \emph{expected sarsa}, EPG integrates across the action when estimating the gradient, instead of relying only on the action in the sampled trajectory. We establish a new \emph{general policy gradient theorem}, of which the stochastic and deterministic policy gradient theorems are special cases. We also prove that EPG reduces the variance of the gradient estimates without requiring deterministic policies and, for the Gaussian case, with no computational overhead. Finally, we show that it is optimal in a certain sense to explore with a Gaussian policy such that the covariance is proportional to $e^H$, where $H$ is the scaled Hessian of the critic with respect to the actions. We present empirical results confirming that this new form of exploration substantially outperforms DPG with the Ornstein-Uhlenbeck heuristic in four challenging MuJoCo domains.
\end{abstract}

\section{Introduction}
\emph{Policy gradient} methods \citep{sutton2000policy, peters2006policy, peters2008reinforcement, silver2014deterministic}, which optimise policies by gradient ascent, have enjoyed great success in reinforcement learning problems with large or continuous action spaces. The archetypal algorithm optimises an \emph{actor}, i.e., a policy, by  following a policy gradient that is estimated using a \emph{critic}, i.e., a value function.

The policy can be stochastic or deterministic, yielding \emph{stochastic policy gradients} (SPG) \citep{sutton2000policy} or \emph{deterministic policy gradients} (DPG) \citep{silver2014deterministic}. The theory underpinning these methods is quite fragmented, as each approach has a separate policy gradient theorem guaranteeing the policy gradient is unbiased under certain conditions.

Furthermore,  both approaches have significant shortcomings. For SPG,  variance in the gradient estimates means that many trajectories are usually needed for learning.  Since gathering trajectories is typically expensive, there is a great need for more sample efficient methods.

 DPG's use of deterministic policies mitigates the problem of variance in the gradient but raises other difficulties.  The theoretical support for DPG is limited since it assumes a critic that approximates $\nabla_a Q$ when in practice it approximates $Q$ instead.  In addition, DPG learns \emph{off-policy}\footnote{We show in this paper that, in certain settings, off-policy DPG is equivalent to EPG, our on-policy method.}, which is undesirable when we want learning  to take the cost of exploration into account. More importantly, learning off-policy necessitates designing a suitable \emph{exploration policy}, which is difficult in practice.  In fact, efficient exploration in DPG is an open problem and most applications simply use independent Gaussian noise or the Ornstein-Uhlenbeck heuristic \citep{uhlenbeck1930theory,lillicrap2015continuous}.

In this paper, we propose a new approach called \emph{expected policy gradients} (EPG) that unifies policy gradients in a way that yields both theoretical and practical insights. Inspired by \emph{expected sarsa} \citep{sutton1998reinforcement,vanseijen:adprl09}, the main idea is to integrate across the action selected by the stochastic policy when estimating the gradient, instead of relying only on the action selected during the sampled trajectory.

EPG enables two theoretical contributions. First, we establish a number of equivalences between EPG and DPG, among which is a new \emph{general policy gradient theorem}, of which the stochastic and deterministic policy gradient theorems are special cases. Second, we prove that EPG reduces the variance of the gradient estimates without requiring deterministic policies and, for the Gaussian case, with no computational overhead over SPG.

EPG also enables a practical contribution: a principled exploration strategy for continuous problems. We show that it is optimal in a certain sense to explore with a Gaussian policy such that the covariance is proportional to $e^H$, where $H$ is the scaled Hessian of the critic with respect to the actions. We present empirical results confirming that this new approach to exploration substantially outperforms DPG with Ornstein-Uhlenbeck exploration in four challenging MuJoCo domains.

\section{Background}
\label{sec:bg}
A \emph{Markov decision process} is a tuple $(S,A,R,p,p_0,\gamma)$ where $S$ is a set of states, $A$ is a set of actions (in practice either $A=\mathbb{R}^d$ or $A$ is finite), $R(s,a)$ is a reward function, $p(s' \mid a,s)$ is a transition kernel, $p_0$ is an initial state distribution, and $\gamma \in [0,1)$ is a discount factor. A policy $\pi(a \mid s)$ is a distribution over actions given a state. We denote trajectories as $\tau^\pi = (s_0,a_0,r_0,s_1,a_1,r_1,\dots)$, where $s_0 \sim p_0$, $a_t \sim \pi(\cdot \mid s_{t-1})$ and $r_t$ is a sample reward. A policy $\pi$ induces a Markov process with transition kernel $p_\pi(s' \mid s) = \int_a d \pi (a \mid s) p(s' \mid a, s)$ where we use the symbol $d \pi (a \mid s)$ to denote Lebesgue integration against the measure $\pi (a \mid s)$ where $s$ is fixed. We assume the induced Markov process is ergodic with a single invariant measure defined for the whole state space.  The value function is $V^\pi = \exs{\tau}{\sum_i \gamma_i r_i}$ where actions are sampled from $\pi$. The $Q$-function is $Q^\pi(a \mid s) = \excs{R}{r}{s,a} + \gamma \excs{p(s\ \mid s)}{V^\pi(s')}{s}$ and the advantage function is $A^\pi(a \mid s) = Q^\pi(a \mid s) - V^\pi(s)$. An optimal policy maximises the total return $J = \int_s dp_0(s) V^\pi(s)$. Since we consider only on-policy learning with just one current policy, we drop the $\pi$ super/subscript where it is redundant.

If $\pi$ is parameterised by $\theta$, then \emph{stochastic policy gradients} (SPG) \citep{sutton2000policy, peters2006policy, peters2008reinforcement} perform gradient ascent on $\nabla J$, the gradient of $J$ with respect to $\theta$ (gradients without a subscript are always with respect to $\theta$). For stochastic policies, we have:
\begin{gather}
\label{spg-update} \textstyle
\nabla J = \int_s d \rho(s) \int_a d \pi(a \mid s) \nabla \log \pi(a \mid s) (Q(a,s) + b(s)),
\end{gather}
where $\rho$ is the discounted-ergodic occupancy measure, defined in the supplement, and $b(s)$ is a baseline, which can be any function that depends on the state but not the action, since $\int_a d \pi(a \mid s) \nabla \log \pi(a \mid s) b(s) = 0$. Typically, \eqref{spg-update} is approximated from samples from a trajectory $\tau$ of length $T$:
\begin{gather}
\label{spg-samples}\textstyle
\hat{\nabla} J = \sum_{t=0}^{T} \gamma^t \nabla \log \pi(a_t \mid s_t) (\hat{Q}(s_t, a_t) + b(s_t)).
\end{gather}
If the policy is deterministic (we denote it $\pi(s)$), we can use \emph{deterministic policy gradients} \citep{silver2014deterministic} instead:
\begin{gather}
    \label{dpg-update}\textstyle
    \nabla J = \int_s d \rho(s) \nabla \pi(s) \nabla_a Q (a = \pi(s), s).
\end{gather}
This update is then approximated using samples:
\begin{gather}
    \label{dpg-samples}\textstyle
    \hat{\nabla} J = \sum_{t=0}^{T} \gamma^t \nabla \pi(s) \nabla_a \hat{Q} (a = \pi(s_t), s_t).
\end{gather}
Since the policy is deterministic, the problem of exploration is addressed using an external source of noise, typically modeled using a zero-mean Ornstein-Uhlenbeck (OU) process \citep{uhlenbeck1930theory, lillicrap2015continuous} parametrized by $\psi$ and $\sigma$:
\begin{gather}
    \label{ou-noise}
    n_i \leftarrow - n_{i-1} \psi + \mathcal{N}(0,\sigma I) \quad a \sim \pi(s) + n_i.
\end{gather}
In \eqref{spg-samples} and \eqref{dpg-samples}, $\hat{Q}$ is a \emph{critic} that approximates $Q$ and can be learned by \emph{sarsa} \citep{rummery1994line, sutton1996generalization}:
\begin{align}
\label{sarsa}
\hat{Q}(s_t, a_t) \leftarrow &\hat{Q}(s_t, a_t) \; + \nonumber \\ \textstyle
&\; \textstyle \alpha \big[ r_{t+1} + \gamma \hat{Q}(s_{t+1}, a_{t+1}) - \hat{Q}(s_t, a_t) \big].
\end{align}

Alternatively, we can use \emph{expected sarsa} \citep{sutton1998reinforcement,vanseijen:adprl09}, which marginalises out $a_{t+1}$, the distribution over which is specified by the known policy, to reduce the variance in the update:
\begin{align}
\label{esarsa} \textstyle
\hat{Q}(&s_t, a_t) \leftarrow \hat{Q}(s_t, a_t) \; + \nonumber \\ \textstyle
 &\; \textstyle \; \alpha \big[ r_{t+1} + \gamma \int_a d \pi(a \mid s) \hat{Q}(s_{t+1}, a) - \hat{Q}(s_t, a_t) \big].
\end{align}
We could also use advantage learning \citep{baird1995residual} or LSTDQ \citep{lagoudakis2003least}. If the critic's function approximator is \emph{compatible}, then the actor, i.e., $\pi$, converges \citep{sutton2000policy}.

Instead of learning $\hat{Q}$, we can set $b(s) = - V(s)$ so that $Q(a,s) + b(s) = A(s,a)$ and then use the TD error $\delta(r,s',s) = r + \gamma V(s') - V(s)$ as an estimate of $A(s,a)$ \citep{bhatnagar2008incremental}:
\begin{gather}
\label{spg-samples-td} \textstyle
\hat{\nabla} J = \sum_{t=0}^{T} \gamma^t \nabla \log \pi(a_t \mid s_t) (r + \gamma \hat{V}(s') - \hat{V}(s)),
\end{gather}
where $\hat{V}(s)$ is an approximate value function learned using any policy evaluation algorithm. \eqref{spg-samples-td} works because $\exc{\delta(r,s',s)}{a,s} = A(s,a)$, i.e., the TD error is an unbiased estimate of the advantage function. The benefit of this approach is that it is sometimes easier to approximate $V$ than $Q$ and that the return in the TD error is unprojected, i.e., it is not distorted by function approximation. However, the TD error is noisy, introducing variance in the gradient.

To cope with this variance, we can reduce the learning rate when the variance of the gradient would otherwise explode, using, e.g., \emph{Adam} \citep{kingma2014adam}, \emph{natural policy gradients} \citep{kakade2002natural, amari1998natural, peters2008natural}, the adaptive step size method \citep{pirotta2013adaptive} or \emph{Newton's method} \citep{furmston2012unifying, parisi2016multi}. However, this results in slow learning when the variance is high. One can also use PGPE, which replaces the stochastic policy with a distribution over deterministic policies \citep{sehnke2010parameter}. However, PGPE precludes updating the current policy during the episode and makes it difficult to explore efficiently.

We can also eliminate all variance caused by the policy at the cost of making the policy deterministic and using the DPG update, which usually necessitates performing off-policy exploration. EPG, presented below, reduces to DPG in many useful cases, while providing a principled way to explore and also allowing for stochastic policies.

Yet another way to eliminate variance in the actor is not to have an actor at all, instead selecting actions soft-greedily with respect to $\hat{Q}$ learned using sarsa. This is trivial for discrete actions and can also be done with a one-step Newton's method for $Q$-functions that are quadric in the actions \citep{gu2016continuous}.

\section{Expected Policy Gradients}
In this section, we propose \emph{expected policy gradients} (EPG).

\subsection{Main Algorithm}
  First, we introduce $I^Q_\pi(s)$ to denote the inner integral in \eqref{spg-update}:
\begin{align}
\nabla J & = \int_s d\rho(s) \underbrace{\int_a d \pi(a \mid s) \nabla \log \pi(a \mid s) (Q(a, s) + b(s))}_{I^Q_\pi(s)} \nonumber \\ & = \int_s d \rho(s) I^Q_\pi(s).
\end{align}
This suggests a new way to write the approximate gradient:
\begin{gather}
\label{spg-samples-i}
\hat{\nabla} J = \sum_{t=0}^{T} \underbrace{\gamma^t \hat{I}^{\hat{Q}}_\pi(s_t)}_{g_t},
\end{gather}
where $\hat{I}^{\hat{Q}}_\pi(s)$ is some approximation to $I^{\hat{Q}}_\pi(s) = \int_a d \pi(a \mid s) \nabla \log \pi(a \mid s) (\hat{Q}(a, s) + b(s))$.  This approach makes explicit that one step in estimating the gradient is to evaluate an integral to estimate $I^{\hat{Q}}_\pi(s)$. The main insight behind EPG is that, given a state, $I^{\hat{Q}}_\pi(s)$ is expressed fully in terms of known quantities. Hence we can manipulate it analytically to obtain a formula or we can just compute the integral using any numerical quadrature if an analytical solution is impossible.

SPG as given in \eqref{spg-samples} performs this quadrature using a simple one-sample Monte Carlo method. However, relying on such a method is unnecessary. In fact, the actions used to interact with the environment need not be used at all in the evaluation of $\hat{I}_\pi^Q(s)$ since $a$ is a bound variable in the definition of $I_\pi^Q(s)$. The motivation is thus similar to that of expected sarsa but applied to the actor's gradient estimate instead of the critic's update rule. EPG, shown in Algorithm \ref{alg-epg}, uses \eqref{spg-samples-i} to form a policy gradient algorithm that repeatedly estimates $\hat{I}_\pi^Q(s)$ with an integration subroutine.

\begin{algorithm}[ht]
\begin{algorithmic}[1]
 \State $s \gets s_0$, $t \gets 0$
 \State initialise optimiser, initialise policy $\pi$ parametrised by $\theta$
\While{not converged}
 \label{ln-a-i} \State $g_t \gets \gamma^t$ \textsc{do-integral}($\hat{Q}, s, \pi_\theta $) \\ \Comment{$g_t$ is the estimated policy gradient as per \eqref{spg-samples-i}}
 \State $\theta \gets  \theta \; + \;  $optimiser.\textsc{update}$(g_t) $
 \State $a \sim \pi(\cdot, s)$
 \State $s',r \gets $ simulator.\textsc{perform-action}(a)
 \State $\hat{Q}$.\textsc{update}($s,a,r,s'$)
 \State $t \gets t + 1$
 \State $s \gets s'$
\EndWhile

\end{algorithmic}
\caption{Expected Policy Gradients} \label{alg-epg}
\end{algorithm}
EPG has benefits even when an analytical solution is not possible: if the action space is low dimensional, numerical quadrature is cheap; if it is high dimensional, it is still often worthwhile to balance the expense of simulating the system with the cost of quadrature. Actually, even in the extreme case of expensive quadrature but cheap simulation, the limited resources available for quadrature could still be better spent on EPG with smart quadrature than SPG with simple Monte Carlo. One of the motivations of DPG was precisely that the simple one-sample Monte-Carlo quadrature implicitly used by SPG often yields high variance gradient estimates, even with a good baseline. To see why, consider Figure \ref{fig-mc-p} (left). A simple Monte Carlo method evaluates the integral by sampling one or more times from $\pi(a \mid s)$ (blue) and evaluating $\nabla_\mu \log \pi(a \mid s) Q(a, s)$ (red) as a function of $a$. A baseline can decrease the variance by adding a multiple of $\nabla_\mu \log \pi(a \mid s)$ to the red curve, but the problem  remains that the red curve has high values where the blue curve is almost zero. Consequently, substantial variance persists, whatever the baseline, even with a simple linear $Q$-function, as shown in Figure \ref{fig-mc-p} (right).  DPG addressed this problem for deterministic policies but EPG extends it to stochastic ones.

\subsection{Relationship to Other Methods}

EPG has some similarities with VINE sampling \citep{schulman2015trust}, which uses an (intrinsically noisy) Monte Carlo quadrature with many samples.\footnote{VINE sampling also differs from EPG by performing independent rollouts of $Q$, requiring a simulator with reset.} However, the example in Figure \ref{fig-mc-p} shows that even with a computationally expensive many-sample Monte Carlo method, the problem of variance remains, regardless of the baseline.

EPG is also related to variance minimisation techniques that interpolate between two estimators, e.g., \cite[Eq. 7]{gu2016q} is similar to Corollary \ref{linear-gpg}. However, EPG uses a quadric (not linear) approximation to the critic, which is crucial for exploration. Furthermore, it completely eliminates variance in the inner integral, as opposed to just reducing it.

The idea behind EPG was also independently and concurrently developed as Mean Actor Critic \citep{2017arXiv170900503A}, though only for discrete actions and without a supporting theoretical analysis.

\begin{figure}
    \centering
        \begin{subfigure}{0.225\textwidth}
            \includegraphics[width=\textwidth]{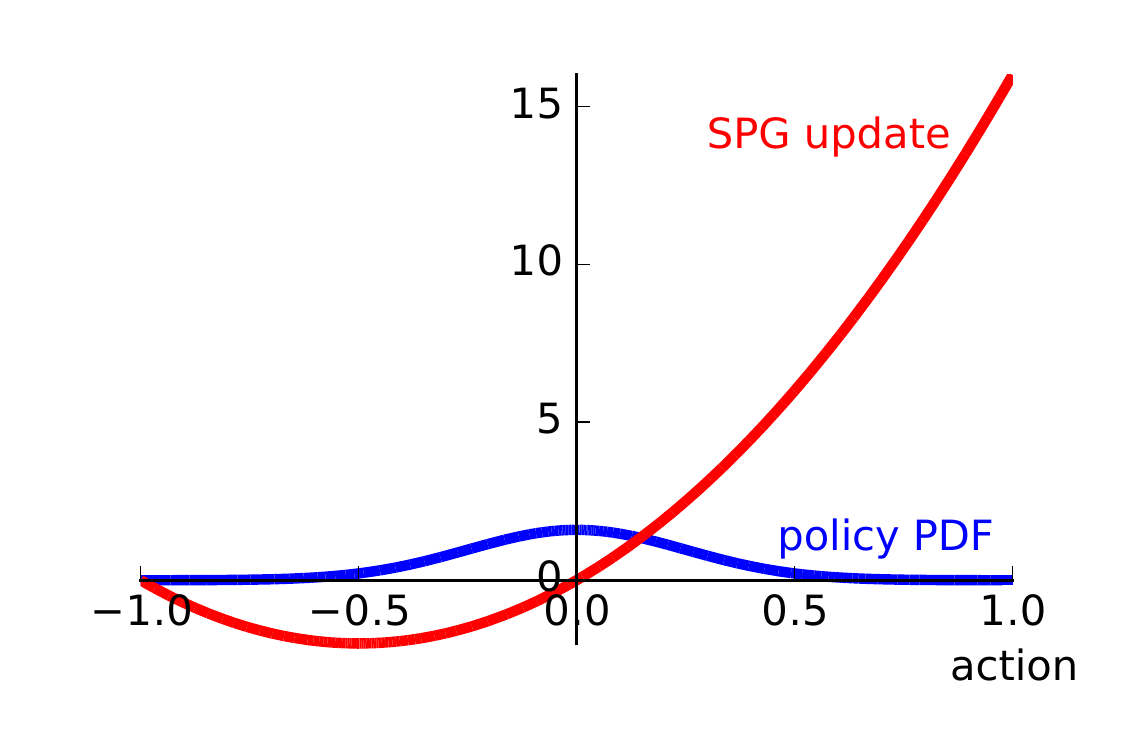}
        \end{subfigure}
        ~
        \begin{subfigure}{0.225\textwidth}
            \includegraphics[width=\textwidth]{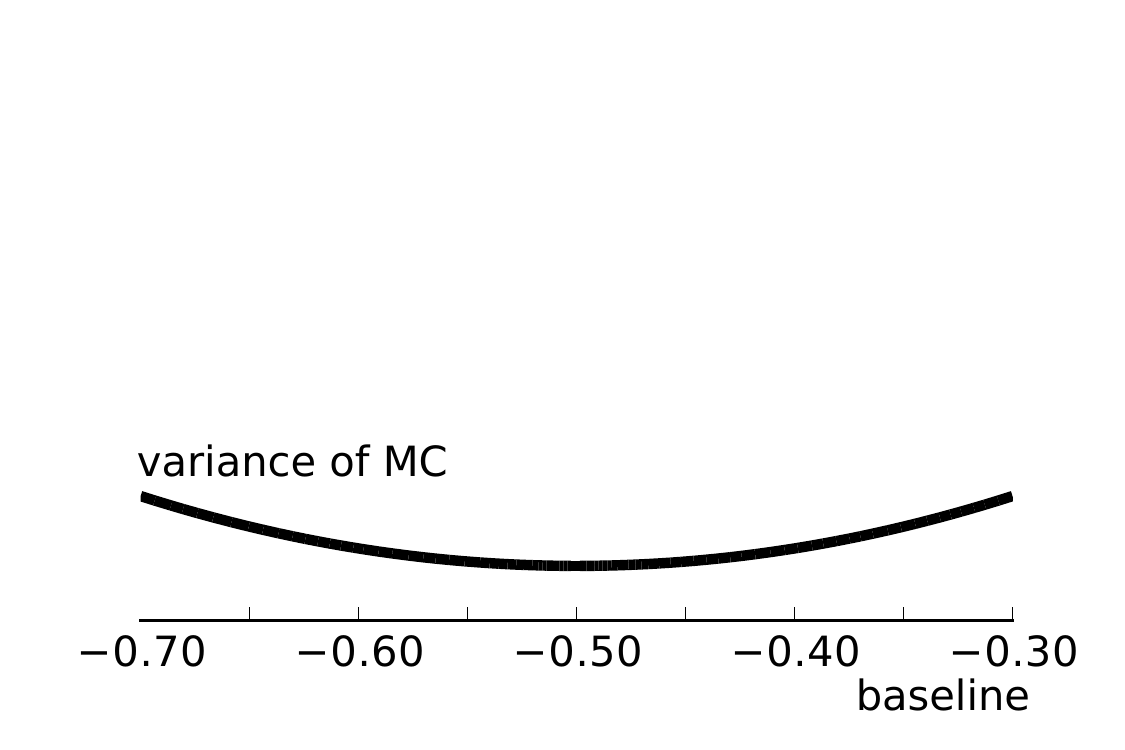}
        \end{subfigure}
        \caption{At left, $\pi(a \mid s)$ for a Gaussian policy with $\mu = \theta = 0$ at a given state and constant $\sigma^2$ (blue) and the SPG update $\nabla_\theta \log \pi(a \mid s) Q(a, s)$ (in red), obtained for $Q = \frac12 + \frac12 a$. At right, the variance of a simple single-sample Monte Carlo estimator as a function of the baseline.  In a simple multi-sample Monte Carlo method, the variance would go down as the number of samples.}
        \label{fig-mc-p}
    \end{figure}

\subsection{Gaussian Policies}

EPG is particularly useful when we make the common assumption of a Gaussian policy: we can then perform the integration analytically under reasonable conditions. We show below (see Lemma \ref{lem-agpg}) that the update to the policy mean computed by EPG is equivalent to the DPG update. Moreover, a simple formula for the covariance can be derived (see Lemma \ref{lem-hexplore}). Algorithms \ref{alg-epg-fixpoint} and \ref{alg-p-quad} show the resulting special case of EPG, which we call \emph{Gaussian policy gradients} (GPG).

\begin{algorithm}[ht]
    \begin{algorithmic}[1]
     \State $s \gets s_0$, $t \gets 0$
 \State initialise optimiser
    \While{not converged}
     \State $g_t \gets \gamma^t$ \textsc{do-integral-Gauss}($\hat{Q}, s, \pi_\theta $)
     \State $\theta \gets  \theta \; + \;  $optimiser.\textsc{update}$(g_t) $ \\ \Comment{policy parameters $\theta$ are updated using gradient}
     \State $\Sigma_s \gets $  \textsc{get-covariance}($\hat{Q}, s, \pi_\theta $) \\ \Comment{$\Sigma_s$ computed from scratch}
     \State $a \sim \pi(\cdot \mid s)$ \Comment {$\pi(\cdot \mid s)= N(\mu_s, \Sigma_s)$}
     \State $s',r \gets $ simulator.\textsc{perform-action}(a)
     \State $\hat{Q}$.\textsc{update}($s,a,r,s'$)
     \State $t \gets t + 1$
     \State $s \gets s'$
    \EndWhile

    \end{algorithmic}
    \caption{Gaussian Policy Gradients} \label{alg-epg-fixpoint}
\end{algorithm}
\begin{algorithm}[ht]
    \begin{algorithmic}[1]
    \Function{do-integral-Gauss}{$\hat{Q}, s, \pi_\theta $}
    \State $I^Q_{\pi(s), \mu_s} \leftarrow (\nabla \mu_s) \nabla_a \hat{Q}(a = \mu_s,s)$ \Comment{Use Lemma \ref{lem-gpg} }
    \State \Return $I^Q_{\pi(s), \mu_s}$
    \EndFunction
    \\
    \Function{get-covariance}{$\hat{Q}, s, \pi_\theta $}
    \State $H \leftarrow$ \textsc{compute-Hessian}($\hat{Q}(\mu_s,s)$)
    \State \Return $\sigma_0^2 e^{cH}$ \Comment{Use Lemma \ref{lem-hexplore}}
    \EndFunction
\end{algorithmic}
\caption{Gaussian Integrals} \label{alg-p-quad}
\end{algorithm}

Surprisingly, GPG is on-policy but nonetheless fully equivalent to DPG, an off-policy method, with a particular form of exploration.  Hence, GPG, by specifying the policy's covariance, can be seen as a derivation of an exploration strategy for DPG. In this way, GPG addresses an important open question.  As we show later, this leads to improved performance in practice.

The computational cost of GPG is small: while it must store a Hessian matrix $H(a, s) = \nabla^2_a \hat{Q}(a,s)$, its size is only $d \times d$, where $A=\mathbb{R}^d$, which is typically small, e.g., $d=6$ for HalfCheetah-v1. This Hessian is the same size as the policy's covariance matrix, which any policy gradient must store anyway, and should not be confused with the Hessian with respect to the parameters of the neural network, as used with Newton's or natural gradient methods \citep{peters2008natural, furmston2016approximate}, which can easily have thousands of entries. Hence, GPG obtains EPG's variance reduction essentially for free.

\section{Analysis}
\label{sec:analysis}

In this section, we analyse EPG, showing that it unifies SPG and DPG, that $\hat{I}_\pi^Q(s)$ can often be computed analytically, and that EPG has lower variance than SPG.

\subsection{General Policy Gradient Theorem}
We begin by stating our most general result, showing that EPG can be seen as a generalisation of both SPG and DPG. To do this, we first state a new general policy gradient theorem. We use the shorthand $\nabla$ without a subscript to denote the gradient with respect to policy parameters $\theta$.
\begin{theorem}[General Policy Gradient Theorem]
If $\pi(\cdot,s)$ is a normalised Lebesgue measure for all $s$, then
\[
\nabla J = \int_s d\rho(s) \underbrace{ \left[ \nabla V(s) - \int_a d \pi(a,s) \nabla Q(a,s) \right]}_{I_G(s)}.
\]
\label{th-gpgt}
\end{theorem}
\begin{proof}
We begin by expanding the following expression.
\begin{align*}
 & \textstyle \int_s \textstyle d\rho(s) \int_a d\pi(a,s) \nabla Q(a,s) \\
&\scriptstyle = \int_s d\rho(s) \int_a d\pi(a,s) \nabla (R(a,s) + \gamma \int_{s'} d p(s' \mid s, a) V(s'))  \\
&\scriptstyle = \int_s d \rho(s) \int_a d \pi(a,s)  (\underbrace{\scriptstyle \nabla R(a,s)}_{\scriptstyle0} + \gamma \int_{s'} d p(s' \mid s, a) \nabla V(s')) \\
&\textstyle = \gamma \int_s d \rho(s) \int_{s'} d p_\pi(s' \mid s) \nabla V(s') \\
&\textstyle = \int_s d \rho(s) \nabla V(s) -  \underbrace{\textstyle \int_s d p_0(s) \nabla V(s)} _{\nabla J} \\
&\textstyle = \int_s d \rho(s) \nabla V(s) -  \nabla J.
\end{align*}
The first equality follows by expanding the definition of $Q$ and the penultimate one follows from Lemma B (in the supplement). Then the theorem follows by rearranging terms.
\end{proof}
The crucial benefit of Theorem \ref{th-gpgt} is that it works for all policies, both stochastic and deterministic, unifying previously separate derivations for the two settings. To show this, in the following two corollaries, we use Theorem \ref{th-gpgt} to recover the \emph{stochastic policy gradient theorem} \citep{sutton2000policy}  and the \emph{deterministic policy gradient theorem} \citep{silver2014deterministic}, in each case by introducing additional assumptions to obtain a formula for $I_G(s)$ expressible in terms of known quantities.
\begin{corollary}[Stochastic Policy Gradient Theorem]
\label{cor-spg}
If $\pi(\cdot \mid s)$ is differentiable, then
\begin{align*} \textstyle
\nabla J &= \textstyle \int_s d \rho(s) I_G(s) \\
&= \textstyle \int_s d \rho(s) \int_a d \pi(a \mid s) \nabla \log \pi(a \mid s) Q(a, s).
\end{align*}
\end{corollary}
\begin{proof}
We obtain the following by expanding $\nabla V $.
\begin{align*}
\nabla V &= \nabla \textstyle \int_a d \pi(a,s) Q(a,s) = \\
& \textstyle \int_a da (\nabla \pi(a,s)) Q(a,s) +  \int_a d \pi(a,s) (\nabla Q(a,s))
\end{align*}
We obtain $I_G(s) = \int_a d \pi(a \mid s) \nabla \log \pi(a \mid s) Q(a, s) = I^Q_\pi(s)$ by plugging this into the definition of $I_G(s)$. We obtain $\nabla J$ by invoking  Theorem \ref{th-gpgt} and plugging in the above expression for $I_G(s)$.
\end{proof}

We now recover the DPG update introduced in \eqref{dpg-update}.
\begin{corollary}[Deterministic Policy Gradient Theorem]
If $\pi(\cdot \mid s)$ is a Dirac-delta measure (i.e., a deterministic policy) and $Q(\cdot,s)$ is differentiable, then
\[
\textstyle \nabla J = \int_s d \rho(s) I_G(s) =  \int_s d \rho(s) \nabla \pi(s) \nabla_a Q(a, s).
\]
\end{corollary}
\begin{proof}
We begin by obtaining an expression for $I_G(s)$.
\begin{align*}
I_G(s) &= \textstyle \nabla V(s) - \int_a d \pi(a,s) \nabla Q(a,s) \\
&= \textstyle \nabla V(s) - \gamma \int_{s'} d p_\pi(s' \mid s) \nabla V(s') \\
&= \textstyle \nabla \pi(s) \nabla_a Q(a, s).
\end{align*}
Here, the second equality follows by expanding the definition of $Q$ and the third follows from an established deterministic policy gradient result \cite[Supplement, Eq.\ 1]{silver2014deterministic}.  We can then obtain $\nabla J$ by invoking Theorem \ref{th-gpgt} and plugging in the above expression for $I_G(s)$.
\end{proof}

These corollaries show that the choice between deterministic and stochastic policy gradients is fundamentally a choice of quadrature method. Hence, the empirical success of DPG relative to SPG \citep{silver2014deterministic, lillicrap2015continuous} can be understood in a new light. In particular, it can be attributed, not to a fundamental limitation of stochastic policies (indeed, stochastic policies are sometimes preferred), but instead to superior quadrature. DPG integrates over Dirac-delta measures, which is known to be easy, while SPG typically relies on simple Monte Carlo integration. Thanks to EPG, a deterministic approach is no longer required to obtain a method with low variance.

We add as a sideline that since Theorem \ref{th-gpgt} can be written as $I_G(s) = \nabla V(s) - \gamma \int_{s'} d p_\pi(s' \mid s) \nabla V(s')$, which involves the derivatives of value functions, GPG resembles \emph{value gradients} \citep{heess2015learning}. However, in our case, we are learning $\nabla J$ directly and do not perform recursive estimation of $\nabla V$ as value gradient methods do.

\subsection{Analytical Quadrature - Gaussian Policy}
We now derive a lemma supporting GPG.
\begin{lemma}[Gaussian Policy Gradients]
\label{lem-gpg}
If the policy is Gaussian, i.e. $\pi(\cdot \vert s) \sim \mathcal{N}(\mu_s,\Sigma_s)$ with $\mu_s$ and $\Sigmasqr_s$ parametrised by $\theta$, where $\Sigmasqr_s$ is symmetric and $\Sigmasqr_s \Sigmasqr_s = \Sigma_s$ and the critic is of the form $Q(a,s) = a^\top A(s) a + a^\top B(s) + \text{const}$ where $A(s)$ is symmetric for every $s$, then $I^Q_\pi(s) =  I^Q_{\pi(s), \mu_s} + I^Q_{\pi(s), \Sigmasqr_s}$, where the mean and covariance components are given by $ I^Q_{\pi(s), \mu_s} = (\nabla \mu_s) ( 2 A(s) \mu + B(s))$ and $I^Q_{\pi(s), \Sigmasqr_s} = (\nabla \Sigmasqr_s) 2 A(s) \Sigmasqr_s $.
\end{lemma}
See Lemma \ref{lem-gpg} in the supplement for proof of this result. While Lemma \ref{lem-gpg} requires the critic to be quadric in the actions, this assumption is not very restrictive since the coefficients $B(s)$ and $A(s)$ can be arbitrary continuous functions of the state, e.g., a neural network.

\subsection{Arbitrary Critics}
If $Q$ does not meet the conditions of Lemma \ref{lem-gpg}, we can  approximate $Q$ with a quadric function in the neighbourhood of the policy mean. This approximation is motivated by two arguments. First, in MDPs that model physical systems with reasonable reward functions, $Q$ is fairly smooth. Second, policy gradients are a local, incremental method anyway -- since the policy mean changes slowly, the values of $Q$ for actions far from the policy mean are usually not relevant for the current update.
\begin{corollary}[Approximate Gaussian Policy Gradients with an Arbitrary Critic]
\label{lem-agpg}
If the policy is Gaussian, i.e. $\pi(\cdot \vert s) \sim \mathcal{N}(\mu_s,\Sigma_s)$ with $\mu_s$ and $\Sigmasqr_s$ parametrised by $\theta$ as in Lemma \ref{lem-gpg} and any critic $Q(a,s)$ doubly differentiable with respect to actions for each state, then $ I^Q_{\pi(s), \mu_s} \approx (\nabla \mu_s) \nabla_a Q(a = \mu_s, s) $ and $I^Q_{\pi(s), \Sigmasqr_s} \approx (\nabla \Sigmasqr_s) H(\mu_s, s) \Sigmasqr_s $, where $H(\mu_s, s)$ is the Hessian of $Q$ with respect to $a$, evaluated at $\mu_s$ for a fixed $s$.
\end{corollary}
\begin{proof}
We begin by approximating the critic (for a given $s$) using the first two terms of the Taylor expansion of $Q$ in $\mu_s$.
\begin{align*}
    \textstyle Q(a, s) &  \textstyle \approx Q(\mu_s, s) + (a - \mu_s)^\top \nabla_a Q (a = \mu_s, s) \\
    & \textstyle \quad \quad \quad \quad \quad \quad \quad  + \frac{1}{2} (a - \mu_s)^\top H(\mu_s, s) (a - \mu_s) \\
    & \scriptstyle = \frac{1}{2} a^\top H(\mu_s, s) a + a^\top \left( \nabla_a Q (a = \mu_s, s) - H(\mu_s, s) \mu_s \right)  + \text{const}.
\end{align*}
Because of the series truncation, the function on the righthand side is quadric and we can then use Lemma \ref{lem-gpg}:
\begin{align*}
    \scriptstyle I^Q_{\pi(s), \mu_s} &\scriptstyle = \nabla \mu_s (2 \frac{1}{2}  H(\mu_s, s) \mu_s + \nabla_a Q (a = \mu_s, s) - H(\mu_s, s) \mu_s) \\
    &\scriptstyle = \nabla \mu_s \nabla_a Q (a = \mu_s, s) \\
    \scriptstyle I^Q_{\pi(s), \Sigmasqr_s} &\scriptstyle = (\nabla {\Sigmasqr_s}) (\frac{1}{2} 2 H(\mu_s, s) \Sigmasqr_s) = (\nabla {\Sigmasqr_s}) H(\mu_s, s) \Sigmasqr_s.
\end{align*}
\end{proof}
To actually obtain the Hessian, we could use automatic differentiation to compute it analytically. Alternatively, we can observe that, if the critic really is quadric, we can just read off the coefficients of the quadric term directly. Therefore, we can approximate the Hessian by generating a number of random action-values around $\mu_s$, computing the $Q$ values, and (locally) fitting a quadric. This process is typically more computationally expensive than automatic differentiation but has the advantage of working with ReLU networks (where the true Hessian is zero but we still have a kind of global curvature after smoothing) and leveraging more information from the critic (since the evaluation is at more than one point).

\subsection{Linear GPG}
We now state a consequence of Lemma \ref{lem-gpg} for the case when the critic $Q$ is linear in the actions, i.e., the quadric term is always zero.
\begin{corollary}[Linear Gaussian Policy Gradients]
    \label{linear-gpg}
    If the policy is Gaussian, i.e., $\pi(\cdot \vert s) \sim \mathcal{N}(\mu_s,\Sigma_s)$ with $\mu_s$ parametrised by $\theta$ and the critic is of the form $Q(a \mid s) = a^{\top} B(s) +  \text{const}$, then $I^Q_\pi(s) = (\nabla \mu_s) B(s)$. Moreover, it is unnecessary to parameterise $\Sigmasqr_s$ since the policy gradient w.r.t.\ to $\Sigmasqr_s$ is zero (i.e., a linear $Q$-function does not give any information about the exploration covariance).
\end{corollary}
We make Corollary \ref{linear-gpg} explicit for two reasons. First, it is useful for showing an equivalence between DPG and EPG (see below). Second, it may actually be  useful for a non-trivial class of physical systems: if the time-sampling frequency is high enough (which implies acting in small steps), the critic is effectively only used to say if a small step one way is preferable to small step the other way -- a linear property.

\subsection{Equivalences between EPG and DPG}
The update for the policy mean obtained in Corollary \ref{lem-agpg} is  the same as the DPG update, linking the two methods:
\[
I^Q_\pi(s) = (\nabla \mu_s) \nabla_a Q(a = \mu_s,s).
\]

We now formalise the equivalences between EPG and DPG. First, on-policy GPG with a linear critic (or an arbitrary critic approximated by the first term in the Taylor expansion) is equivalent to DPG with a Gaussian exploration policy where the covariance stays the same. This follows from Corollary \ref{linear-gpg}. Second, on-policy GPG with a quadric critic (or an arbitrary critic approximated by the first two terms in the Taylor expansion) is equivalent to DPG with a Gaussian exploration policy where the covariance is computed using the update (where $\alpha_n$ is a sequence of step-sizes):
\begin{gather}
    \label{eq-cov-upd}
    \Sigmasqr_s \leftarrow \Sigmasqr_s + \alpha_n H(s) \Sigmasqr_s.
\end{gather}
This follows from Corollary \ref{lem-agpg}. Third, and most generally, for any critic at all (not necessarily quadric), DPG is a kind of EPG for a particular choice of quadrature (using a Dirac measure). This follows from Theorem \ref{th-gpgt}.

Surprisingly, this means that DPG, normally considered to be off-policy, can also be seen as on-policy when exploring with Gaussian noise. Furthermore, the compatible critic for DPG \citep{silver2014deterministic} is indeed linear in the actions. Hence, this relationship holds whenever DPG uses a compatible critic.\footnote{The notion of compatibility of a critic is different for stochastic and deterministic policy gradients.} Furthermore, Lemma \ref{lem-gpg} lends new legitimacy to the common practice of replacing the critic required by the DPG theory, which approximates $\nabla_a Q$, with one that approximates $Q$ itself, as done in SPG and EPG.

\subsection{Exploration using the Hessian}
The second equivalence given above suggests that we can include the covariance in the actor network and learn it along with the mean. However, another option is to compute it from scratch at each iteration by analytically computing the result of applying \eqref{eq-cov-upd} infinitely many times.

\begin{lemma}[Exploration Limit]
\label{lem-hexplore}
The iterative procedure defined by the equation $\Sigmasqr_s \leftarrow \Sigmasqr_s + \alpha H(s) \Sigmasqr_s$ applied $n$ times using the diminishing learning rate $\alpha = 1/n$ converges to $\Sigmasqr_s \propto e^{H(s)} $ as $n \rightarrow \infty$.
\end{lemma}
\begin{proof}
Consider the sequence $(\Sigmasqr_s)_0 = \sigma_0 I$, $(\Sigmasqr_s)_n = (\Sigmasqr_s)_{n-1} + \alpha H(s) (\Sigmasqr_s)_{n-1}$. Expanding out the recursion, the $n$-th element of the sequence is given as:
\[
    (\Sigmasqr_s)_n = (I + \alpha H(s))^n (\Sigmasqr_s)_0.
\]
We diagonalise the Hessian as $H(s) = U \Lambda U^\top$ for some orthonormal matrix $U$ and obtain the following expression for $(\Sigmasqr_s)_n$.
\[
    (\Sigmasqr_s)_n = (I + \alpha U \Lambda U^\top)^n (\Sigmasqr_s)_0 =
    U (I + \alpha \Lambda )^n U^\top (\Sigmasqr_s)_0
\]
Since we have $\lim_{n \rightarrow \infty} (1 + \frac1n \lambda)^n = e^\lambda$ for each diagonal entry of $\Lambda$, we plug $\alpha = \frac 1n$ and obtain the identity:
\[
    \lim_{n \rightarrow \infty} (\Sigmasqr_s)_n =  U e^\Lambda U^\top (\Sigmasqr_s)_0 = \sigma_0 e^{H(s)}.
\]
\end{proof}
The practical implication of Lemma \ref{lem-hexplore} is that, in a policy gradient method, it is justified to use Gaussian exploration with covariance proportional to $e^{cH}$ for some reward scaling constant $c$. Thus by exploring with (scaled) covariance $e^{cH}$, we obtain a principled alternative to the Ornstein-Uhlenbeck heuristic defined in \eqref{ou-noise}. Our results below show that it also performs much better in practice.

Lemma \ref{lem-hexplore} has an intuitive interpretation. If $H(s)$ has a large positive eigenvalue $\lambda$, then $\hat{Q}(s,\cdot)$ has a sharp minimum along the corresponding eigenvector, and the corresponding eigenvalue of $\Sigma$ is $e^\lambda$, i.e., also large. The result is a large exploration bonus along that direction, enabling the algorithm to leave local minima. Conversely, if $\lambda$ is negative, then $\hat{Q}(s,\cdot)$ has a maximum and so $e^\lambda$ is small, since exploration is not needed.

\subsection{Variance Analysis}
We now prove that for any policy, the EPG estimator of \eqref{spg-samples-i} has lower variance than the SPG estimator of \eqref{spg-samples}.

\begin{lemma}
If for all $s \in S$, the random variable $\nabla \log \pi(a \mid s) \hat{Q}(s, a)$ where $a \sim \pi(\cdot \vert s)$ has nonzero variance, then
\begin{gather*} \scriptstyle
\scriptstyle \vars{\tau}{\scriptstyle \sum_{t=0}^{\infty} \gamma^t \nabla \log \pi(a_t \mid s_t) (\hat{Q}(s_t, a_t) + b(s_t))} \scriptstyle >
 \scriptstyle \vars{\tau}{  \sum_{t=0}^{\infty} \gamma^t I^{\hat{Q}}_\pi(s_t)  } .
\end{gather*}
\label{lem-var}
\end{lemma}

The proof is deferred to the supplement (see Lemma \ref{lem-var} there). Lemma \ref{lem-var}'s assumption is reasonable since the only way a random variable $\nabla \log \pi(a \mid s) \hat{Q}(s, a)$ could have zero variance is if it were the same for all actions in the policy's support (except for sets of measure zero), in which case optimising the policy would be unnecessary. Since we know that both the estimators of  \eqref{spg-samples} and \eqref{spg-samples-i} are unbiased, the estimator with lower variance has lower MSE.

\subsection{Extension to Entropy Regularisation}
On-policy SPG sometimes includes an entropy term in the gradient in order to aid exploration by making the policy more stochastic. The gradient of the differential entropy\footnote{For discrete action spaces, the same derivation with integrals replaced by sums holds for the entropy.} $\mathcal{H}(s)$ of the policy at state $s$ is defined as follows.
\begin{align*}
    \scriptstyle  - & \scriptstyle \nabla \mathcal{H}(s) = \scriptstyle \nabla \int_a d \pi(a \vert s) \log \pi (a \vert s) \\ & \scriptstyle = \scriptstyle \int_a  da \nabla \pi(a \vert s) \log \pi (a \vert s) +  \int_a  d \pi(a \vert s) \nabla \log \pi (a \vert s) \\
 & \scriptstyle = \scriptstyle \int_a da \nabla \pi(a \vert s) \log \pi (a \vert s) +  \int_a  d \pi(a \vert s) \frac{1}{ \pi (a \vert s)} \nabla \pi(a \vert s) \\ & \scriptstyle = \scriptstyle \int_a da \nabla \pi(a \vert s) \log \pi (a \vert s) +  \nabla \scriptstyle \underbrace{\scriptstyle \int_a  d \pi(a \vert s)}_{1} \\
 & \scriptstyle =  \scriptstyle \int_a  da \nabla \pi(a \vert s) \log \pi (a \vert s) = \scriptstyle  \int_a  d \pi(a \vert s) \nabla \log \pi(a \vert s) \log \pi (a \vert s).
\end{align*}
Typically, we weight the entropy update with the policy gradient update:
\begin{align*}
I_G^E(s) \textstyle & = I_G(s) + \alpha \nabla \mathcal{H}(s) \\
&= \textstyle \int_a  d \pi(a \vert s) \nabla \log \pi(a \vert s) (Q(a,s) - \alpha \log \pi (a \vert s)).
\end{align*}
This equation makes clear that performing entropy regularisation is equivalent to using a different critic with $Q$-values shifted by $\alpha \log \pi (a \vert s)$; this holds for both SPG and EPG.

\begin{table}[th]
    \centering
        \begin{tabular}{ l | p{2.0cm} p{2.0cm} }
        Domain & $\hat{\sigma}_\text{DPG}$ & $\hat{\sigma}_\text{EPG}$ \\
        \hline
        HalfCheetah-v1 & 1336.39 \newline \tiny [1107.85, 1614.51] & 1056.15 \newline \tiny [875.54, 1275.94] \\
        InvertedPendulum-v1 & 291.26 \newline \tiny [241.45, 351.88] & 0.00 \newline \tiny n/a \\
        Reacher2d-v1 & 1.22 \newline \tiny [0.63, 2.31] & 0.13 \newline \tiny [0.07, 0.26] \\
        Walker2d-1 & 543.54 \newline \tiny [450.58, 656.65] & 762.35 \newline \tiny [631.98, 921.00] \\
        \end{tabular}
    \caption{Estimated standard deviation (mean and 90\% interval) across runs after learning.}
    \label{tab-std}
\end{table}

\section{Experiments}
 While EPG has many potential uses, we focus on empirically evaluating one particular application: exploration driven by the Hessian exponential (as introduced in Algorithm \ref{alg-epg-fixpoint} and Lemma \ref{lem-hexplore}), replacing the standard Ornstein-Uhlenbeck (OU) exploration in continuous action domains. To this end, we applied EPG to four domains modelled with the MuJoCo physics simulator \citep{todorov2012mujoco}: HalfCheetah-v1, InvertedPendulum-v1, Reacher2d-v1 and Walker2d-v1 and compared its performance to DPG and SPG.

In practice, EPG differed from deep DPG \citep{lillicrap2015continuous, silver2014deterministic} only in the exploration strategy, though their theoretical underpinnings are different. The hyperparameters for DPG and those of EPG that are not related to exploration were taken from an existing benchmark \citep{islam2017reproducibility, brockman2016openai}. The exploration hyperparameters for EPG were $\sigma_0^2 = 0.2$ and $c = 1.0$ where the exploration covariance is $\sigma_0^2 e^{cH}$. These values were obtained using a grid search from the set $\{0.2,0.5,1\}$ for $\sigma_0^2$ and $\{ 0.5, 1.0, 2.0 \}$ for $c$ over the HalfCheetah-v1 domain. Since $c$ is just a constant scaling the rewards, it is reasonable to set it to $1.0$ whenever reward scaling is already used. Hence, our exploration strategy has just one hyperparameter $\sigma_0^2$ as opposed specifying a pair of parameters (standard deviation and mean reversion constant) for OU. We used the same learning parameters for the other domains. For SPG\footnote{We tried learning the covariance for SPG but the covariance estimate was unstable; no regularisation hyperparameters we tested matched SPG's performance with OU even on the simplest domain.}, we used OU exploration and a constant diagonal covariance of $0.2$ in the actor update (this approximately corresponds to the average variance of the OU process over time). The other parameters for SPG are the same as for the rest of the algorithm. For the learning curves, we obtained 90\% confidence intervals around the learning curves. The learning curves show results of independent evaluation runs which used actions generated by the policy mean without any exploration noise.

The results (Figure \ref{fig-res}) show that EPG's exploration strategy yields much better performance than DPG with OU. Furthermore, SPG does poorly, solving only the easiest domain (InvertedPendulum-v1) reasonably quickly, achieving slow progress on  HalfCheetah-v1, and failing entirely on the other domains. This is not surprising DPG was introduced  precisely to solve the problem of high variance SPG estimates on this type of problem. In InvertedPendulum-v1, SPG initially learns quickly, outperforming the other methods. This is because noisy gradient updates provide a crude, indirect form of exploration that happens to suit this problem. Clearly, this is  inadequate for more complex domains: even for this simple domain it leads to subpar performance late in learning.

\begin{figure}[tb]
\centering
    \begin{subfigure}{0.22\textwidth}
        \includegraphics[width=\textwidth]{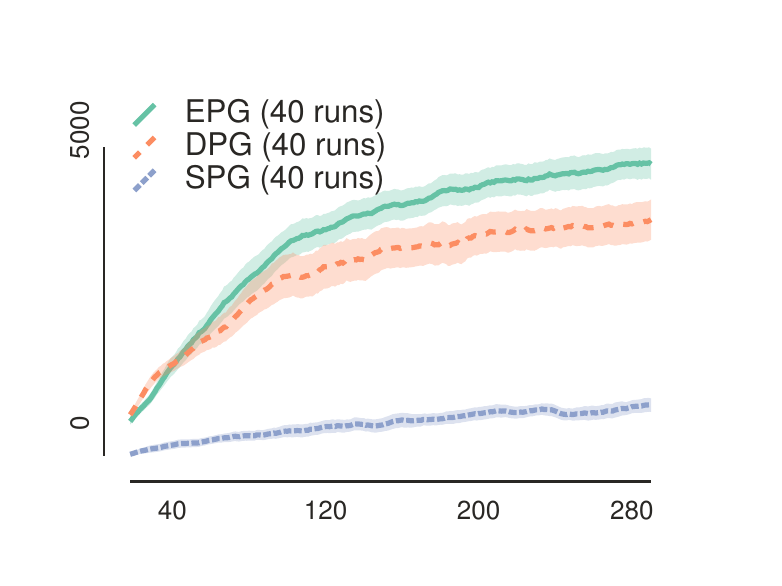}
    \end{subfigure}
    ~
    \begin{subfigure}{0.22\textwidth}
        \includegraphics[width=\textwidth]{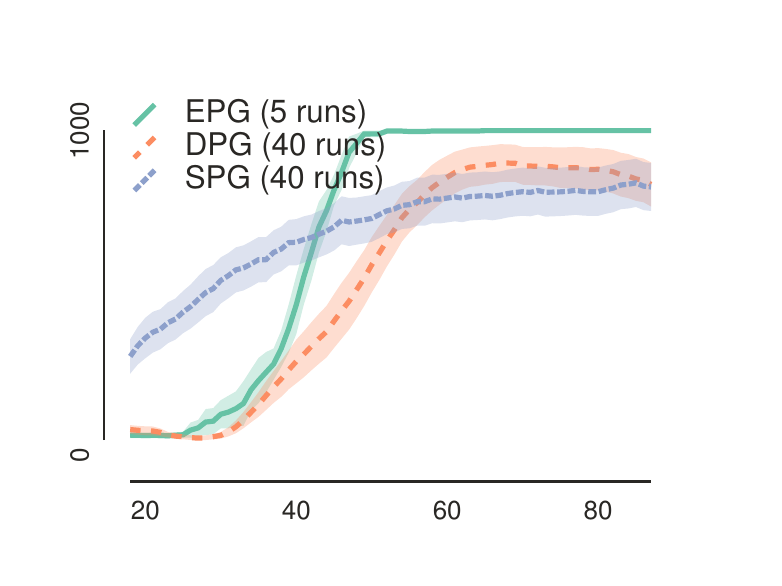}
    \end{subfigure}
    \\
    \begin{subfigure}{0.22\textwidth}
        \includegraphics[width=\textwidth]{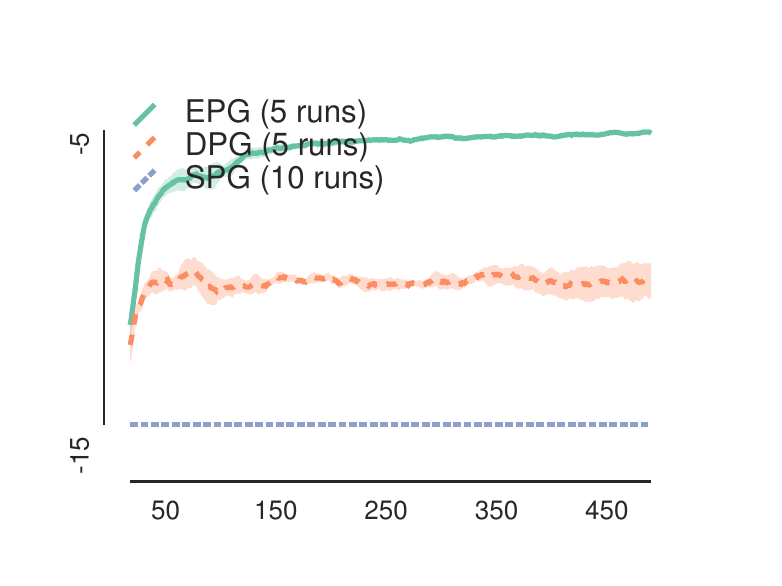}
    \end{subfigure}
    ~
    \begin{subfigure}{0.22\textwidth}
        \includegraphics[width=\textwidth]{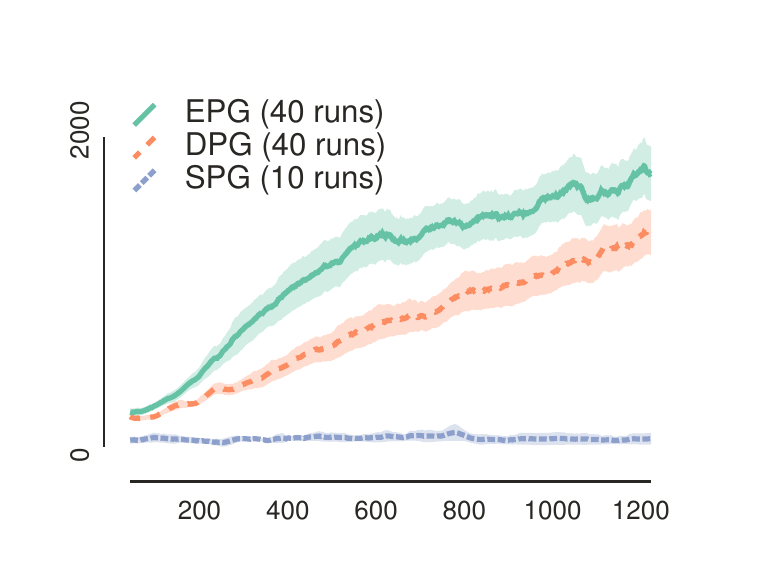}
    \end{subfigure}
    \caption{Learning curves (mean and 90\% interval) for HalfCheetah-v1 (top left), InvertedPendulum-v1 (top right), Reacher2d-v1 (bottom left, clipped at -14) and Walker2d-v1 (bottom right). The number of independent training runs is in parentheses. Horizontal axis is scaled in thousands of steps.}
    \label{fig-res}
\end{figure}

\begin{figure}[h]
    \centering
    \includegraphics[width=0.14\textwidth]{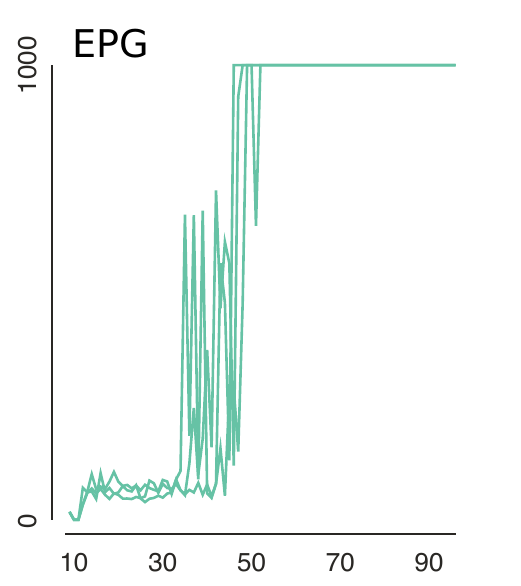}
    \includegraphics[width=0.14\textwidth]{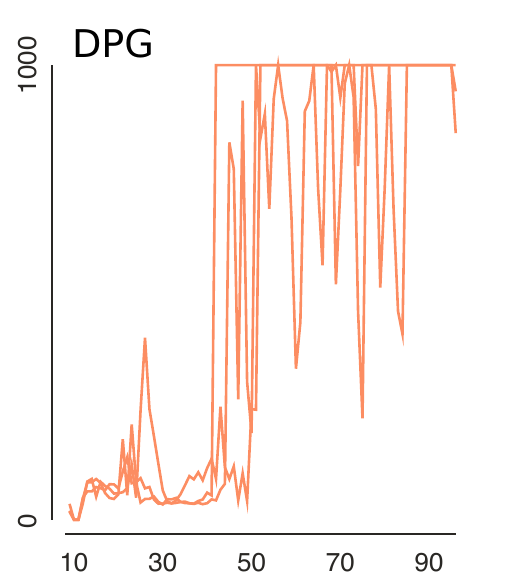}
    \includegraphics[width=0.14\textwidth]{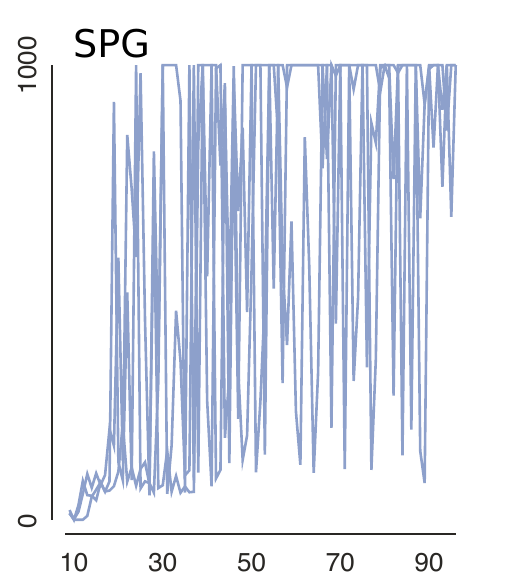}
    \caption{Three runs for EPG (left), DPG (middle) and SPG (right) for the InvertedPendulum-v1 domain, demonstrating that EPG shows much less unlearning.}
    \label{fig-res-unlearn}
\end{figure}

In addition, EPG typically learns more consistently than DPG with OU. In two tasks, the empirical standard deviation across runs of EPG ($\hat{\sigma}_\text{EPG}$) was substantially lower than that of DPG ($\hat{\sigma}_\text{DPG}$) at the end of learning, as shown in Table \ref{tab-std}. For the other two domains, the confidence intervals around the empirical standard deviations for DPG and EPG were too wide to draw conclusions.

Surprisingly, for InvertedPendulum-v1, DPG's learning curve declines late in learning. The reason can be seen in the individual runs shown in Figure \ref{fig-res-unlearn}: both DPG and SPG suffer from severe unlearning. This unlearning cannot be explained by exploration noise since the evaluation runs just use the mean action, without exploring. Instead, OU exploration in DPG may be too coarse, causing the optimiser to exit good optima, while SPG unlearns due to noise in the gradients. The noise also helps speed initial learning, as described above, but this does not transfer to other domains. EPG avoids this problem by automatically reducing the noise when it finds a good optimum, i.e., a Hessian with large negative eigenvalues.

\section{Conclusions}

This paper proposed a new policy gradient method called \emph{expected policy gradients} (EPG), that integrates across the action selected by the stochastic policy. We used EPG to prove a new general policy gradient theorem subsuming the stochastic and deterministic policy gradient theorems. We also showed that, under certain realistic conditions, the quadrature required by EPG can be performed analytically, allowing DPG with principled exploration. We presented empirical results confirming that this application of EPG outperforms DPG and SPG on four domains.

\section{Acknowledgements}
This project has received funding from the European Research Council (ERC) under the European Union's Horizon 2020 research and innovation programme (grant agreement number 637713).
{
\fontsize{9.0pt}{10.5pt} \selectfont
\bibliography{Bibliography-File}
\bibliographystyle{aaai}
}

\onecolumn
\newpage

\section{Supplement}
We first provide formal proofs for certain statements invoked by our paper. We then provide a brief discussion of the use of a learning rate that diminished in the trajectory length in the computation of the covariance.

\section{Proofs}
\label{sec-proofs}
First, we prove two lemmas concerning the discounted-ergodic measure $\rho(s)$ which have been implicitly realised for some time but as far as we could find, never proved explicitly.

\begin{definition}[Time-dependent occupancy]
\label{def-tdo}
\begin{gather*}
p(s \mid t=0 ) = p_0(s) \\
p(s' \mid t=i+1 ) = \int_s p(s' \mid s) p(s \mid t=i ) \quad \text{for} \quad i \geq 0
\end{gather*}
\end{definition}

\begin{definition}[Truncated trajectory]
Define the trajectory truncated after $N$ steps as $\tau^N = (s_0,a_0,r_0,s_1,a_1,r_1,\dots,s_N)$.
\end{definition}

\begin{observation}[Expectation wrt. truncated trajectory]
Since $\tau_N = (s_0,s_1,s_2,\dots,s_N)$ is associated with the density $ \prod_{i=0}^{N-1} p(s_{i+1} \mid s_i) p_0(s_0)$, we have that
\begin{align*} \textstyle
& \textstyle \exs{\tau_N} {\sum_{i=0}^{N} \gamma^i f(s_i)} = \\
 & \quad = \textstyle \int_{s_0, s_1, \dots, s_N} \left (\prod_{i=0}^{N-1} p(s_{i+1} \mid s_i) \right) p_0(s_0) \left( \sum_{i=0}^{N} \gamma^i f(s_i) \right) ds_0 ds_1 \dots ds_N= \\
 & \quad= \textstyle \sum_{i=0}^{N} \int_{s_0, s_1, \dots, s_N} \left ( p_0(s_0) \prod_{i=0}^{N-1} p(s_{i+1} \mid s_i) \right) \gamma^i f(s_i) ds_0 ds_1 \dots ds_N = \\
 & \quad= \textstyle \sum_{i=0}^{N} \int_s p(s \mid t=i) \gamma^i f(s) ds
\end{align*}

for any function $f$.

\end{observation}

\begin{definition}[Expectation with respect to infinte trajectory]
For any bounded function $f$, we have
\begin{gather*}
\exs{\tau} {\sum_{i=0}^{\infty} \gamma^i f(s_i)} \triangleq \lim_{N \rightarrow \infty} \exs{\tau_N} {\sum_{i=0}^{N} \gamma^i f(s_i)}.
\end{gather*}
Here, the sum on the left-hand side is part of the symbol being defined.
\end{definition}

\begin{observation}[Property of expectation with respect to infinte trajectory]
\label{obs-eprop}
\begin{align*} \textstyle
\exs{\tau} {\sum_{i=0}^{\infty} \gamma^i f(s_i)}
 &= \textstyle \lim_{N \rightarrow \infty} \exs{\tau_N} {\sum_{i=0}^{N} \gamma^i f(s_i)} = \\
 &= \textstyle \lim_{N \rightarrow \infty} \sum_{i=0}^{N} \int_s p(s \mid t=i) \gamma^i f(s) d s  = \\
 &= \sum_{i=0}^\infty \int_s d p(s \mid t=i) \gamma^i f(s)
\end{align*}
for any bounded function $f$.
\end{observation}

\begin{definition}[Discounted-ergodic occupancy measure $\rho$]
\label{def-deom}
\[
\rho(s) = \sum_{i=0}^{\infty} \gamma^i p(s \mid t=i)
\]
\end{definition}
The measure $\rho$ is not normalised in general. Intuitively, it can be thought of as `marginalising out' the time in the system dynamics.

\begin{lemma}[Discounted-ergodic property]
For any bounded function $f$:
\[
\int_s \rho(s) f(s) = \exs{\tau}{\sum_{i=0}^{\infty}  \gamma^i f(s_i)}.
\]
\end{lemma}
\begin{proof}
\[
\exs{\tau}{\sum_{i=0}^{\infty} \gamma^i f(s_i) } = \sum_{i=0}^{\infty} \gamma^i \int_s p(s \mid t=i) f(s) ds = \int_s \underbrace{ \left[ \sum_{i=0}^{\infty} \gamma^i p(s \mid t=i) \right ]}_{\rho(s)} f(s) ds
\]
Here, the first equality follows from Observation \ref{obs-eprop}.
\end{proof}
This property is useful since the expression on the left can be easily manipulated while the expression on the right can be estimated from samples using Monte Carlo.

\begin{lemma}[Generalised eigenfunction property]
\label{gep-lemma}
For any bounded function $f$:
\[
\gamma \int_s d \rho(s) \int_{s'} d p(s' \mid s) f(s') = \left( \int_s d \rho(s) f(s) \right) - \left ( \int_s d p_0(s) f(s) \right)
\]
\end{lemma}
\begin{proof}
\begin{align*}
    \textstyle
\gamma \int_s d\rho(s) \int_{s'} d p(s' \mid s) f(s') &= \textstyle \gamma \sum_{i=0}^{\infty} \gamma^i \int_{s,s'} p(s \mid t=i) p(s' \mid s) f(s') ds ds'= \\
&= \textstyle \sum_{i=0}^{\infty} \gamma^{i+1} \int_{s'} dp(s' \mid t=i+1) f(s') \\
&= \textstyle \sum_{i=1}^{\infty} \gamma^{i} \int_{s'} dp(s' \mid t=i) f(s') \\
&= \textstyle \left ( \sum_{i=0}^{\infty} \gamma^{i} \int_{s'} dp(s' \mid t=i) f(s') \right) - \left ( \int_s dp_0(s) f(s) \right) \\
&= \textstyle \left( \int_s d\rho(s) f(s) \right) - \left ( \int_s dp_0(s) f(s) \right)
\end{align*}
Here, the first equality follows form definition \ref{def-deom}, the second one from definition \ref{def-tdo}. The last equality follows again from definition \ref{def-deom}.
\end{proof}

\begin{definition}[Markov Reward Process]
A Markov Reward Process is a tuple $(p,p_0,R,\gamma)$, where $p(s' \vert s)$ is a transition kernel, $p_0$ is the distribution over initial states, $R(\cdot \vert s)$ is a reward distribution conditioned on the state and $\gamma$ is the discount constant.
\end{definition}

An MRP can be thought of as an MDP with a fixed policy and dynamics given by marginalising out the actions $p_\pi(s' \mid s) = \int_a d \pi (a \mid s) p(s' \mid a, s)$. Since this paper considers the case of one policy, we abuse notation slightly by using the same symbol $\tau$ to denote trajectories including actions, i.e. $(s_0,a_0,r_0,s_1,a_1,r_1,\dots)$ and without them $(s_0,r_0,s_1,r_1,\dots)$.

\begin{lemma}[Second Moment Bellman Equation]
\label{smbe-lemma}
Consider a Markov Reward Process $(p,p_0,X,\gamma)$ where $p(s' \mid s)$ is a Markov process and $X( \cdot \mid s)$ is some probability density function\footnote{Note that while $X$ occupies a place in the definition of the MRP usually called `reward distribution', we are using the symbol $X$, not $R$ since we shall apply the lemma to $X$es which are constructions distinct from the reward of the MDP we are solving.}.  Denote the value function of the MRP as $V$. Denote the second moment function $S$ as
\[
S(s) = \excs{\tau}{ \left( \sum_{t=0}^{\infty} \gamma^t x_t \right)^2 }{s_0 = s} \quad x_t \sim X(\cdot \mid s_t).
\]
Then $S$ is the value function of the MRP: $(p, p_0, u, \gamma^2)$, where $u(s)$ is a deterministic random variable given by
\[
u(s) = \vars{X( x \mid s)}{x} + \left(\exs{X( x \mid s)}{x}\right)^2 + 2 \gamma \exs{X( x \mid s)}{x} \exs{p( s' \mid s)}{V(s')} .
\]
\end{lemma}
\begin{proof}
\begin{align*} \textstyle
S(s) &= \textstyle \excs{\tau}{ \left( x_0 + \sum_{t=1}^{\infty} \gamma^t x_t \right)^2 }{s_0 = s} \\
&= \textstyle \excs{\tau}{ x_0^2 + 2 x_0 \left( \sum_{t=1}^{\infty} \gamma^t x_t \right) +  \left( \sum_{t=1}^{\infty} \gamma^t x_t \right)^2}{s_0 = s} \\
&= \textstyle \underbrace{ \textstyle \excs{\tau}{ x_0^2 }{s_0 = s} + \excs{\tau}{ 2 x_0 \left( \sum_{t=1}^{\infty} \gamma^t x_t \right)}{s_0 = s}}_{u(s)} +  \underbrace{\excs{\tau}{ \textstyle \left( \sum_{t=1}^{\infty} \gamma^t x_t \right)^2}{s_0 = s}}_{\gamma^2  \exs{p( s' \mid s)}{S(s')}}
\end{align*}
This is exactly the Bellman equation of the MRP $(p, p_0, u, \gamma^2)$. The theorem follows since the Bellman equation uniquely determines the value function.
\end{proof}

\begin{observation}[Dominated Value Functions]
\label{dom-o}
Consider two Markov Reward Processes $(p,p_0,X_1,\gamma)$ and $(p,p_0,X_2,\gamma)$, where $p(s' \mid s)$ is a Markov process (common to both MRPs) and $X_1(s)$, $X_2(s)$ are some deterministic random variables meeting the condition $X_1(s) \leq X_2(s)$ for every $s$. Then the value functions $V_1$ and $V_2$ of the respective MRPs satisfy $V_1(s) \leq V_2(s)$ for every $s$. Moreover, if we have that $X_1(s) < X_2(s)$ for all states, then the inequality between value functions is strict.
\end{observation}
\begin{proof}
Follows trivially by expanding the value function as a series and comparing series elementwise.
\end{proof}

We now move our attention to prove the Gaussian Policy Gradients lemma.

\begin{customlemma}{1}[Gaussian Policy Gradients]

    \label{lem-gpg-proof}
    If the policy is Gaussian, i.e. $\pi(\cdot \vert s) \sim \mathcal{N}(\mu_s,\Sigma_s)$ with $\mu_s$ and $\Sigmasqr_s$ parametrised by $\theta$, where $\Sigmasqr_s$ is symmetric and $\Sigmasqr_s \Sigmasqr_s = \Sigma_s$ and the critic is of the form $Q(a,s) = a^\top A(s) a + a^\top B(s) + \text{const}$ where $A(s)$ is symmetric for every $s$, then $I^Q_\pi(s) =  I^Q_{\pi(s), \mu_s} + I^Q_{\pi(s), \Sigmasqr_s}$, where the mean and covariance components are given by $ I^Q_{\pi(s), \mu_s} = (\nabla \mu_s) ( 2 A(s) \mu + B(s))$ and $I^Q_{\pi(s), \Sigmasqr_s} = (\nabla \Sigmasqr_s) 2 A(s) \Sigmasqr_s $.
    \end{customlemma}
    \begin{proof}
    
    First, we observe that the critic $Q$ defined in the statement of the lemma does not depend on the policy parameters $\theta$. This is because $Q$ is an approximation to the Q-function maintained by the algorithm as opposed to the true Q-function, which is defined with respect to the policy and does depend on it.

    We can hence move the differentiation outside of the integral, as follows.
    \[
        I^Q_\pi(s) = \nabla \int_a \pi(a \vert s) Q(a,s) da = \nabla \exs{\pi}{Q(a,s)} 
    \]

    We now expand the expectation using a known expression for the expectation of a a quadratic form:
    \[
        \exs{\pi}{Q(a,s)} = \trace (A(s) \Sigma ) + \mu^\top A(s) \mu + B(s)^\top \mu.
    \]

    This gives way to the following derivatives.
    \begin{align*}
        \nabla_{\Sigmasqr} \exs{\pi}{Q(a,s)} &= \nabla_{\Sigmasqr} (\trace (A(s) \Sigma ) + \mu^\top A(s) \mu + B(s)^\top \mu) = 2 A(s) \Sigmasqr \\
        \nabla_\mu \exs{\pi}{Q(a,s)} &= \nabla_\mu (\trace (A(s) \Sigma ) + \mu^\top A(s) \mu + B(s)^\top \mu) = 2 A(s) \mu + B(s)
    \end{align*}.

    We now obtain the result by applying chain rule. 
    \[
        I^Q_\pi(s) = I^Q_{\pi(s), \mu_s} + I^Q_{\pi(s), \Sigmasqr_s} = 
        (\nabla \mu) (2 A(s) \mu + B(s)) + (\nabla \Sigmasqr) (2 A(s) \Sigmasqr)
    \] 

\end{proof}

\begin{customlemma}{3}
    If for all $s \in S$, the random variable $\nabla \log \pi(a \mid s) \hat{Q}(s, a)$ where $a \sim \pi(a \vert s)$ has nonzero variance, then
    \begin{align*} \textstyle
    &\vars{\tau}{\textstyle \sum_{t=0}^{\infty} \gamma^t \nabla \log \pi(a_t \mid s_t) (\hat{Q}(s_t, a_t) + b(s_t)) \;\;\; }  > \nonumber \\
    & \textstyle \vars{\tau}{  \sum_{t=0}^{\infty} \gamma^t I^{\hat{Q}}_\pi(s_t)  } .
    \end{align*}
    \end{customlemma}
\begin{proof}
    Both random variables have the same mean so we need only show that:
    \begin{align*}
    & \textstyle \exs{\tau}{\left (\sum_{t=0}^{\infty} \gamma^t \nabla \log \pi(a_t \mid s_t) (\hat{Q}(s_t, a_t) + b(s_t)) \right)^2} > \\
     & \textstyle \exs{\tau}{  \sum_{t=0}^{\infty} \left ( \gamma^t I^{\hat{Q}}_\pi(s_t) \right)^2  }.
    \end{align*}
    We start by applying Lemma \ref{smbe-lemma} to the lefthand side and setting $X = X_1(s_t) = \gamma^t \nabla \log \pi(a_t \mid s_t) (\hat{Q}(s_t, a_t) + b(s_t))$ where $a_t \sim \pi(a_t \vert s_t)$. This shows that $\exs{\tau}{\left (\sum_{t=0}^{\infty} \gamma^t \nabla \log \pi(a_t \mid s_t) (\hat{Q}(s_t, a_t) + b(s_t)) \right)^2}$ is the total return of the MRP $(p, p_0, u_1, \gamma^2)$, where
    \begin{gather*}
    u_1 =
    \vars{X_1( x \mid s)}{x} +
    \left( \exs{X_1( x \mid s)}{x}\right)^2 +
    2 \gamma \exs{X_1( x \mid s)}{x}  \exs{p( s' \mid s)} {V(s')}.
    \end{gather*}
    Likewise, applying Lemma \ref{smbe-lemma} again to the righthand side, instantiating $X$ as a deterministic random variable $X_2(s_t) = I^{\hat{Q}}_\pi(s_t) $, we have that $\exs{\tau}{  \sum_{t=0}^{\infty} \left ( \gamma^t I^{\hat{Q}}_\pi(s_t) \right)^2  }$ is the total return of the MRP $(p, p_0, u_2, \gamma^2)$, where
    \[
    u_2 = \left(\exs{X_2( x \mid s)}{x}\right)^2 + 2 \gamma \exs{X_2( x \mid s)}{x} \exs{p( s' \mid s)}{V(s')}.
    \]
    Note that $\exs{X_1( x \mid s)}{x} = \exs{X_2( x \mid s)}{x}$ and therefore $u_1 \geq u_2$. Furthermore, by assumption of the lemma, the inequality is strict.  The lemma then follows by applying Observation \ref{dom-o}.
    \end{proof}
For convenience, Lemma \ref{lem-var} also assumes infinite length trajectories.  However, this is not a practical limitation since all policy gradient methods implicitly assume trajectories are long enough to be modelled as infinite.  Furthermore, a finite trajectory variant also holds, though the proof is messier.

\section{Remarks on the covariance limit}
\label{rem-rml}
When we obtain $e^{H}$ as the limiting covariance matrix in Lemma 2 of the main paper, there is a slight modelling difficulty: is it justified to use the learning rate of $\frac{1}{n}$, which diminished in the length of the trajectory, as opposed to a small finite number? We observe that the problem of choosing step sizes is, in general, not specific to our method since all policy gradient methods rely on stochastic optimisation and hence work with a diminishing learning rate of some sort. We do note; however, that the step size we use, which is $\frac1n$ for every point in the trajectory, is different from the step size typically used with Robbins-Monro procedure, which is different at each time step. This means that the sum of our step sizes is finite while the sum of the Robbins-Monro step-sizes diverges. Hence our choice of step size does not give the guarantees typically associated with stochastic optimisation. We use the step sequence since it serves as a useful intermediate stage between simply taking \emph{one} PG step of equation \eqref{eq-cov-upd} and using a finite step-sizes, which would mean that the covariance would converge either to zero or diverge to infinity.

\end{document}